\documentclass{article}

\usepackage{microtype}
\usepackage{graphicx}
\usepackage{subfigure}
\usepackage{booktabs} 

\usepackage{amsmath}
\usepackage{amssymb}
\usepackage{amsthm}
\usepackage{amsfonts}
\usepackage{bm}
\usepackage{graphicx}
\usepackage{comment}
\usepackage{bbm}
\usepackage{here}
\usepackage{algorithmic,algorithm} 
\usepackage{caption}
\usepackage{honda_sty}
\usepackage{thmtools,thm-restate}

\newtheorem{lemma}{Lemma}
\newtheorem{theorem}{Theorem}
\newtheorem{definition}{Definition}
\newtheorem{proposition}{Proposition}
\newtheorem{corollary}{Corollary}

\DeclareMathOperator*{\argmin}{arg\,min}
\DeclareMathOperator*{\argmax}{arg\,max}

\newcommand{\ga}{\gamma}
\newcommand{\be}{\beta}

\newcommand{\mut}{\tilde{\mu}}
\newcommand{\muhat}{\hat{\mu}}
\newcommand{\omu}{\overline{\mu}}
\newcommand{\umu}{\underline{\mu}}
\newcommand{\calE}{\mathcal{E}}
\newcommand{\minimize}{\mathrm{minimize\;}}
\newcommand{\maximize}{\mathrm{maximize\;}}  
\newcommand{\subjectto}{\mbox{subject\;to\;}}

\usepackage{hyperref}

\usepackage[accepted]{icml2018-arxiv}

\icmltitlerunning{Good Arm Identification via Bandit Feedback}
\allowdisplaybreaks[1]
\begin{document}

\twocolumn[
\icmltitle{Good Arm Identification via Bandit Feedback}



\icmlsetsymbol{equal}{*}

\begin{icmlauthorlist}
\icmlauthor{Hideaki Kano}{ut,riken}
\icmlauthor{Junya Honda}{ut,riken}
\icmlauthor{Kentaro Sakamaki}{ut}

\icmlauthor{Kentaro Matsuura}{jnj}
\icmlauthor{Atsuyoshi Nakamura}{hu}
\icmlauthor{Masashi Sugiyama}{riken,ut}
\end{icmlauthorlist}

\icmlaffiliation{ut}{University of Tokyo}
\icmlaffiliation{riken}{RIKEN}
\icmlaffiliation{jnj}{Johnson \& Johnson}
\icmlaffiliation{hu}{Hokkaido University}

\icmlcorrespondingauthor{Hideaki Kano}{kano@ms.k.u-tokyo.ac.jp}
\icmlcorrespondingauthor{Junya Honda}{jhonda@k.u-tokyo.ac.jp}

\icmlkeywords{Machine Learning, ICML}

\vskip 0.3in
]



\printAffiliationsAndNotice{\icmlEqualContribution} 

\begin{abstract}
    We consider a novel stochastic multi-armed bandit problem called {\em good arm identification} (GAI),
    where a good arm is defined as an arm with expected reward greater than or equal to a given threshold. 
    GAI is a pure-exploration problem
    that a single agent repeats a process of outputting an arm as soon as it is identified as a good one before confirming the other arms are actually not good.
    The objective of GAI is to minimize the number of samples for each process.
    We find that GAI faces a new kind of dilemma, the {\em exploration-exploitation dilemma of confidence},
    which is different difficulty from the best arm identification.
    As a result,
    an efficient design of algorithms for GAI is quite different from that for the best arm identification.
    We derive a lower bound on the sample complexity of GAI that is tight up to the logarithmic factor $\lo(\log \frac{1}{\delta})$ for acceptance error rate $\delta$.
    We also develop an algorithm whose sample complexity almost matches the lower bound.
    We also confirm experimentally that our proposed algorithm outperforms naive algorithms 
    in synthetic settings based on a conventional bandit problem 
    and clinical trial researches for rheumatoid arthritis.
\end{abstract}

\section{Introduction}
  The stochastic multi-armed bandit (MAB) problem is one of the most fundamental problems 
  for sequential decision-making under uncertainty \citep{Sutton1998}.
  It is regarded as a subfield of reinforcement learning 
  in which an agent aims to acquire a policy to select the best-rewarding action via trial and error.
  In the stochastic MAB problem,
  a single agent repeatedly plays $K$ slot machines called \emph{arms}, 
  where an arm generates a stochastic reward when pulled.
  At each round $t$, the agent pulls arm $i \in [K] = \{1,2,\dots,K\}$
  and then observes an i.i.d.~reward $X_i(t)$ from distribution $\nu_i$ with expectation $\mu_i \in [0,1]$.
 
  One of the most classic MAB formulations is the \emph{cumulative regret minimization} \citep{LAI19854, Auer2002},
  where the agent tries to maximize the cumulative reward
  over the fixed number of trials.
  In this setting, the agent faces the \emph{exploration-exploitation dilemma of reward},
  where the exploration means that the agent pulls an arm other than the currently best arm to find better arms,
  and the exploitation indicates that the agent pulls the currently best arm to increase the cumulative reward.
  The related frameworks can be widely applied to various real-world problems 
  such as clinical trials \citep{ASTIN,seckinumab,choy,article,Liu2017} and personalized recommendations \citep{Tang:2015:PRV:2766462.2767707}.

  Another classic branch of the MAB problem is the \emph{best arm identification} \citep{JMLR:v17:kaufman16a, ICML12-shivaram},
  which is a pure-exploration problem that 
  the agent tries to identify the best arm $a^* = \argmax_{i \in \{1,2,\ldots,K\}} \mu_i$.  
  So far, the conceptual idea of the best arm identification has also been successfully applied to many kinds of real-world problems
  \citep{doi:10.1080/03610918508812467, Schmidt2006, pmlr-v32-zhoub14, pmlr-v51-jun16}.
  Recently, the \emph{thresholding bandit problem} was proposed \citep{Locatelli:2016:OAT:3045390.3045569} as a variant of pure-exploration MAB formulations.
  In the thresholding bandit problem, the agent tries to correctly partition all the $K$ arms into good arms and bad arms,
  where a good arm is defined as an arm whose expected reward is greater than or equal to a given threshold $\xi >0$,
  and a bad arm is defined as an arm whose expected reward is lower than the threshold $\xi$.
  However, in practice, neither correctly partitioning all the $K$ arms nor exactly identifying the very best arm
  is always needed;
  rather, finding some of reasonably good arms as fast as possible is often more useful.

  Take a problem of personalized recommendations for example.
  The objective is to increase our profit by sending direct emails recommending personalized items.
  In this problem, timely recommendation is a key,
  because the best sellers in the past are not necessarily the best sellers in the future.
  Now, there arise three troubles if this problem is formulated as the best arm identification or the thresholding bandit problem.
  First, an inflation of exploration costs could break out
  when the purchase probabilities of the multiple best sellers are much close with each other.
  Although this trouble can be partly relaxed by the $\epsilon$-best arm identification \citep{Even-Dar:2006:AES:1248547.1248586},
  in which an arm with expectation greater than or equal to $\max_{i\in [K]} \mu_i-\epsilon$ is also acceptable, 
  the tolerance parameter $\epsilon$ has to be set very conservatively.
  Second, recommending even the best sellers is not a good idea
  if the ``best'' purchase probability is too small considering the advertising costs.
  Third, it needlessly increases exploration costs 
  to partition all items into good (or profitable) items
  and bad (or not profitable) items,
  if it is enough to find only some good items to increase our profit.
  For the above reasons, the formulation of the personalized recommendation problem 
  as the best arm identification or the thresholding bandit problem is not necessarily effective.
  
  Similar troubles also occur in clinical trials for finding drugs \citep{Kim2011} or for finding appropriate doses of a drug \citep{ASTIN,seckinumab,choy,article,Liu2017}, where
  the number of patients is extremely limited.
  In such a case, it is vitally important to find some drugs or doses with satisfactory effect as fast as possible 
  rather than either to classify all drugs or doses into satisfactory ones and others or to identify the exactly best ones.
  
  In this paper, we propose a new bandit framework named \emph{good arm identification} (GAI),
  where a good arm is defined as an arm whose expected reward is greater than or equal to a given threshold.
  We formulate GAI as a pure-exploration problem in the {\em fixed confidence} setting, 
  which is often considered in conventional pure-exploration problems.
  In the fixed confidence setting, an acceptance error rate $\delta$ is fixed in advance,
  and we minimize the number of pulling arms needed to assure the correctness of the output with probability greater than or equal to $1-\delta$.
  In GAI, a single agent repeats a process of outputting an arm as soon as 
  the agent identifies it as a good one
  with error probability at most $\delta$.
  If it is found that there remain no good arms, then the agent stops working.
  Although the agent does not face the exploration-exploitation dilemma of reward since GAI is a pure-exploration problem,
  the agent suffers from a new kind of dilemma, that is the {\em exploration-exploitation dilemma of confidence},
  where the exploration means that the agent pulls other arms than the currently best one that may be easier to confirm to be good,
  and the exploitation indicates that the agent pulls the currently best arm to increase the confidence on the goodness.

  To address the dilemma of confidence, we propose a Hybrid algorithm for the Dilemma of Confidence (HDoC).
  The sampling strategy of HDoC is based on the upper confidence bound (UCB) algorithm for the cumulative regret minimization \citep{Auer2002},
  and the identification rule (that is, the criterion to output an arm as a good one) of HDoC is based on
  the lower confidence bound (LCB) for the best arm identification \citep{ICML12-shivaram}.  
  In addition, we show that a lower bound on the sample complexity for GAI is $\Omega(\lambda\log \frac{1}{\delta})$,
  and HDoC can find $\lambda$ good arms within $\lo \left( \lambda \log \frac{1}{\delta} + \allowbreak (K-\lambda) \log \log \frac{1}{\delta} \right)$ samples.
  This result suggests that HDoC is superior to naive algorithms based on conventional pure-exploration problems,
  because they require $\lo \left(K \log \frac{1}{\delta} \right)$ samples.

  For the personalized recommendation problem, the GAI approach is more appropriate,
  because the agent can quickly identify good items since 
  the agent only focuses on finding good items rather than identifying the best item (as in the best arm identification) and bad items (as in the thresholding bandit).
  Certainly, there exists a possibility that the recommended item does not possess the best purchase probabilities.
  However, that does not necessarily matter
  when customers' interests and item repositories undergo frequent changes,
  because identifying the exactly best item requires too many samples,
  and thus we cannot do that in practice.
  In addition, thanks to the absolute comparison, not the relative comparison,
  the inflation of exploration costs does not break out even if the purchase probabilities are close to each other,
  and then the agent can refrain from recommending items when the purchase probabilities are too small.
  
  Our contributions can be summarized as four folds.
  First, we formulate a novel pure-exploration problem called GAI and find there is a new kind of dilemma, 
  that is, the exploration-exploitation dilemma of confidence.
  Second, we derive a lower bound for GAI in the {\em fixed confidence} setting.
  Third, we propose the HDoC algorithm and show that an upper bound on the sample complexity of HDoC almost matches the lower bound. 
  Fourth, we experimentally demonstrate that HDoC outperforms two naive algorithms derived from other pure exploration problems
  in synthetic settings based on the thresholding bandit problem \citep{Locatelli:2016:OAT:3045390.3045569}
  and the clinical trial researches for rheumatoid arthritis \citep{seckinumab,choy,article,Liu2017}.

\section{Good Arm Identification}
    In this section, we first formulate GAI
    as a pure exploration problem in the fixed confidence setting.
    Next, we derive a lower bound on the  sample complexity for GAI.
    We give the notation list in Table~\ref{Notation1}.
  
    \subsection{Problem Formulation}
      Let $K$ be the number of arms, $\xi \in (0,1)$ be a threshold and $\delta >0$ be an acceptance error rate.
      Each arm $i \in [K] = \{1,2,\ldots,K\}$ is associated with Bernoulli distribution $\nu_i$ with mean $\mu_i$.
      The parameters $\{\mu_i\}_{i=1}^{K}$ are unknown to the agent.
      We define a good arm as an arm whose expected reward is greater than or equal to threshold $\xi$.
      The number of good arms is denoted by $m$ which is unknown to the agent and,
      without loss of generality, we assume an indexing of the arms such that
      \small
      \begin{equation}\label{model}
        \mu_1 \geq \mu_2 \geq \cdots \geq \mu_m \geq \xi \geq \mu_{m+1} \geq \cdots \geq \mu_K. \nn
      \end{equation}
      \normalsize
      The agent is naturally unaware of this indexing.
      At each round $t$, the agent pulls an arm $a(t) \in [K]$
      and receives an i.i.d.\,reward drawn from distribution $\nu_{a(t)}$.
      The agent outputs an arm when it is identified as a good one.
      The agent repeats this process until there remain no good arms, 
      where the stopping time is denoted by $\tau_\mathrm{stop}$.    
      To be more precise, the agent outputs $\hat{a}_1, \hat{a}_2, \ldots,\hat{a}_{\hat{m}}$ 
      as good arms  (which are different from each other) at rounds $\tau_1, \tau_2, \ldots,\tau_{\hat{m}}$, respectively, where $\hat{m}$
      is the number of arms that the agent outputs as good ones.
      The agent stops working after outputting $\bot$ (NULL) at round $\tau_\mathrm{stop}$
      when the agent finds that there remain no good arms.
      If all arms are identified as good ones, then the agent stops after outputting $\hat{a}_K$ and $\bot$ together at the same round.
      For $\lambda>\hat{m}$ we define
      $\tau_{\lambda}=\tau_{\mathrm{stop}}$.        
      Now, we introduce the definitions of ($\lambda$, $\delta$)-PAC (Probably Approximately Correct) and $\delta$-PAC.   
      \begin{definition}[$(\lambda, \delta)$-PAC]\label{lambdaPAC}
        An algorithm satisfying the following
        conditions is called ($\lambda$, $\delta$)-PAC:
        if there are at least $\lambda$ good arms then
        $\mathbb{P}[\{\hat{m}< \lambda\}\,\cup\,\bigcup_{i\in  \{\hat{a}_1, \hat{a}_2, \ldots, \hat{a}_\lambda \}  }\{\mu_i< \xi\}]\le \de$
        and
        if there are less than $\lambda$ good arms then
        $\mathbb{P}[ \hat{m}\ge \lambda]\le \de$.
      \end{definition}
      \begin{definition}[$\delta$-PAC]\label{deltaPAC}
        An algorithm is called $\delta$-PAC if
        the algorithm is $(\lambda, \delta)$-PAC
        for all $\lambda\in [K]$.
      \end{definition}
  
      The agent aims to minimize $\{ \tau_1, \tau_2, \ldots, \tau_\mathrm{stop} \}$ simultaneously by a $\delta$-PAC algorithm.
      On the other hand, the minimization of $\tau_{\mathrm{stop}}$ corresponds to
      the thresholding bandit if we consider the fixed confidence setting.

      As we can easily see from these definitions,
      the condition for a ($\lambda,\delta$)-PAC algorithm is
      weaker than that for a $\delta$-PAC algorithm.
      Thus, there is a possibility that we can construct a good algorithm
      to minimize $\tau_{\lambda}$ by using a $(\lambda,\delta)$-PAC algorithm
      rather than a $\delta$-PAC algorithm if a specific value of $\lambda$ is considered.
      Nevertheless, we will show that the lower bound on $\tau_{\lambda}$
      for $(\lambda,\delta)$-PAC algorithms can be achieved
      by a $\delta$-PAC algorithm without knowledge of $\lambda$.

      \begin{table}[t]
        \caption{Notation List}
        \vspace{2mm}
        \hrule \vspace{2mm}
        \begin{tabular}{r l }
          $K$      & \   Number of arms. \\
          $m$      & \   Number of good arms (unknown). \\
          $\hat{m}$ & \ Number of arms that the agent outputs \\
                    & \  before outputting  $\bot$  (NULL). \\
          $\delta$ & \   Acceptance error rate. \\
          $\xi$    & \   Threshold determining whether arms\\
            \      & \  are good or not. \\
            $a(t)$ & \   Pulled arm at round $t$. \\
          $\mu_i$  & \   Expected reward of arm $i$ (unknown).\\
          $\hat{\mu}_i(t)$ & \  Empirical mean of the rewards of arm $i$\\
          \           & \  by the end of round $t$. \\ 
          $\hat{\mu}_{i,n}$ & \ Empirical mean of the rewards when \\
          \                 & \ arm $i$ has been pulled $n$ times. \\        
          $N_i(t)$ & \   Number of samples of arm $i$ which has\\
              \    & \   been pulled by the end of round $t$. \\
          $\tau_\lambda$ & \  Round that agent identifies $\lambda$ good arms. \\
          $\tau_\mathrm{stop}$ & \ Round that agent outputs $\bot$ (NULL). \\
        \end{tabular}\label{Notation1}
        \vspace{-2mm}
        \small
        \begin{align}
          \mut_i(t) &=\muhat_i(t)+\sqrt{\frac{\log t}{2N_i(t)}}\nn
          \omu_i(t)&=\muhat_i(t)+\sqrt{\frac{\log(4KN_i^2(t)/\delta)}{2N_i^2(t)}}\nn
          \umu_i(t)&=\muhat_i(t)-\sqrt{\frac{\log(4KN_i^2(t)/\delta)}{2N_i(t)}}\nn
          \De_i&=|\mu_i-\xi| \nn
          \De_{i,j}&=\mu_i-\mu_j\nn
          \Delta&=\min\left\{\min_{i\in[K]}\De_i,\allowbreak\min_{\lambda \in [K-1]}\De_{\la,\la+1}/2\right\}\nn
          n_i&=\frac{1}{(\De_i-\ep)^2}
          \log\pax{ 
          \frac{4\sqrt{K/\de}}{(\De_i-\ep)^2}\log \frac{5\sqrt{K/\de}}{(\De_i-\ep)^2}
          }\n
        \end{align}
        \hrule 
        \normalsize
      \end{table} 
  
    \subsection{Lower Bound on the Sample Complexity}\label{LowerBound}
        We give a lower bound on the sample complexity for GAI.
        This proof is given in Section \ref{proof_lower_bound}.
        {\allowdisplaybreaks[0]
        \begin{theorem}\label{thm_lower}
            Under any $(\lambda, \delta)$-PAC algorithm, if there are $m \ge \la$ good arms, then
            \begin{align}\label{lower}
                \E[\tau_{\la}]
                 &\ge
                 \left(
                     \sum_{i=1}^{\la}\frac{1}{d(\mu_i,\xi)}\log\frac{1}{2\delta}
                 \right)
                 -
                 \frac{m}{d(\mu_{\la},\xi)}\com
            \end{align}
            where $d(x,y) = x\log(x/y)+(1-x)\log((1-x)/(1-y))$ is the binary relative entropy, with convention that $d(0,0)=d(1,1)=0$.
        \end{theorem}
        }
        This lower bound on the sample complexity for GAI is given in terms of top-$\lambda$ expectations $\{\mu_i\}_{i=1}^{\lambda}$.
        In the next section we confirm that this lower bound is tight up to the logarithmic factor $\lo(\log \frac{1}{\delta})$.

\section{Algorithms}
    In this section, we first consider naive algorithms based on other pure-exploration problems.
    Next, we propose an algorithm for GAI and bound its sample complexity from above.
    Pseudo codes of all the algorithms are described in Algorithm \ref{Alg1}.
    These algorithms can be decomposed into two components: a sampling strategy and an identification criterion.
    A sampling strategy is a policy to decide which arm the agent pulls.
    An identification criterion is a policy for the agent to decide whether arms are good or bad.
    All the algorithms adopt the same identification criterion of Lines 5--11 in Algorithm \ref{Alg1},
    which is based on the Lower Confidence Bound (LCB) for the best arm identification \citep{ICML12-shivaram}.
    See Remark 3 at the end of Section 3.2 for other choices of identification criteria.

    \subsection{Naive Algorithms}
        We consider two naive algorithms: the Lower and Upper Confidence Bounds algorithm for GAI (LUCB-G), which is based on the LUCB algorithm for the best arm identification \citep{ICML12-shivaram}
        and the Anytime Parameter-free Thresholding algorithm for GAI (APT-G), which is based on the APT algorithm for the thresholding bandit problem \citep{Locatelli:2016:OAT:3045390.3045569}.
        In both algorithms, the sampling strategy is the same as the
        original algorithms.
        These algorithms sample all arms at the same order $\lo \left(\log \frac{1}{\delta} \right)$.
    \subsection{Proposed Algorithm}
        We propose a Hybrid algorithm for the Dilemma of Confidence (HDoC).  
        The sampling strategy of HDoC is based on the UCB score of the cumulative regret minimization \citep{Auer2002}.
        As we will see later, the algorithm stops within $t=\lo(\log \frac{1}{\delta})$ rounds with high probability.
        Thus, the second term of the UCB score of HDoC in \eqref{score_HDoC} is 
        $\lo \left( \sqrt{ \frac{\log \log (1/\delta)}{N_i (t)} } \right),$
        whereas 
        that of LUCB-G in \eqref{score_LUCB-G} is $\lo \left(  \sqrt{ \frac{\log (1/\delta)}{N_i (t)} } \right)$.
        Therefore,
        the HDoC algorithm pulls the currently best arm more frequently than LUCB-G,
        which means that HDoC puts more emphasis on exploitation than exploration. 

        \begin{algorithm}[htb]
            \caption{ HDoC / LUCB-G / APT-G}
            \begin{algorithmic}[1]
                \STATE {\bf Input:} a threshold $\xi$, an acceptance error rate $\delta$ \\ and a set of arms $\calA \gets [K]$.
                \STATE Pull each arm once.
                \REPEAT        
                    \STATE 
                            HDoC: Pull arm $\hat{a}^* = \argmax_{i \in \calA} \tilde{\mu}_i (t)$ for\\
                            \vspace{-2mm}
                            \small
                            \begin{align}\label{score_HDoC}
                            \tilde{\mu}_i (t) = \hat{\mu}_i (t) + \sqrt{\frac{\log t}{2N_i(t)}} \per
                            \end{align}
                            \normalsize
                            \vspace{-2mm}\\
                            LUCB-G: Pull arm $\hat{a}^* = \argmax_{i \in \calA} \overline{\mu}_i (t)$ for\\ 
                            \vspace{-2mm}
                            \small
                            \begin{align}\label{score_LUCB-G}
                            \overline{\mu}_i (t) = \hat{\mu}_i (t) + \sqrt{\frac{\log(4KN_i^2(t)/\delta)}{2N_i(t)}} \per
                            \end{align}  
                            \normalsize
                            \vspace{-2mm}
                            .\\
                            APT-G: Pull arm $\hat{a}^* = \argmin_{i \in \calA} \beta_i (t)$ for\\ 
                            \vspace{-2mm}
                            \[\beta_i (t) = \sqrt{N_i(t)} \, |\xi - \hat{\mu}_i (t)| \per \] \\
                        \IF{$\underline{\mu}_{\hat{a}^*}(t) = \hat{\mu}_{\hat{a}^*}(t) - \sqrt{\frac{\log(4KN_{\hat{a}^*}^2(t)/\delta)}{2N_{\hat{a}^*}(t)}} \geq \xi $}
                            \STATE Output $\hat{a}^*$ as a good arm.

                            \STATE Delete $\hat{a}^*$ from $\calA$.
                        \ENDIF
                        \IF{$\overline{\mu}_{\hat{a}^*} =  \hat{\mu}_{\hat{a}^*} (t) + \sqrt{\frac{\log(4KN_{\hat{a}^*}^2(t)/\delta)}{2N_{\hat{a}^*}(t)}} < \xi$}
                            \STATE Delete $\hat{a}^*$ from $\calA$.
                        \ENDIF
                \UNTIL{$\overline{\mu}_i < \xi, \, \forall i \in\calA \per$}
            \end{algorithmic} \label{Alg1}
        \end{algorithm}

        The correctness of the output of the HDoC algorithm can be verified by the following theorem, whose proof is given in Appendix \ref{proof_error}.
        \begin{theorem}\label{thm_error}
            The HDoC algorithm is $\delta$-PAC.
        \end{theorem}
        This theorem means that the HDoC algorithm outputs a bad arm
        with probability at most $\delta$.

        Next we give an upper bound on the sample complexity of HDoC.
        We bound the sample complexity in terms of $\De_i=|\mu_i-\xi|$ and $\De_{i,j}=\mu_i-\mu_j$.
        \begin{restatable}{theorem}{upbound}\label{thm_upper}
            Assume that $\De_{\la,\la+1}>0$.
            Then, for any $\la\le m$ and
            $\ep<\min\{\min_{i\in[K]}\De_i,\,\De_{\la,\la+1}/2\}$,
            \begin{align}
                \E[\tau_{\la}]
                &\le
                \sum_{i\in [\la]}n_i
                +
                \sum_{i\in [K]\setminus[\la]}
                \left(
                \frac{\log (K\max_{j\in [K]}n_j )}{2(\De_{\la,i} - 2\epsilon)^2}
                +\de
                n_i\right) \nn
                &\quad+
                \frac{K^{2-\frac{\ep^2}{(\min_{i \in [K]} \De_i-\ep)^2}}}{2\ep^2}
                +
                \frac{K(5+\log\frac{1}{2\epsilon^2})}{4\epsilon^2}\com\nn
                \E[\tau_{\mathrm{stop}}]
                &\le
                \sum_{i\in [K]}n_i+\frac{K}{2\ep^2}\com\n 
            \end{align}
            where
            \begin{align}
                n_i=\frac{1}{(\De_i-\ep)^2}
                \log\pax{
                \frac{4\sqrt{K/\de}}{(\De_i-\ep)^2}\log \frac{5\sqrt{K/\de}}{(\De_i-\ep)^2}
                }\per\n
            \end{align}
        \end{restatable}
        We prove this theorem in Appendix \ref{proof_upper_bound}.
        The following corollary is straightforward from this theorem.
        \begin{corollary}\label{cor_upper} 
            Let $\Delta=\min\{\min_{i\in[K]}\De_i,\allowbreak\min_{\lambda \in [K-1]}\De_{\la,\la+1}/2\}$.
            Then, for any $\la \le m$,
            \begin{align}
                \limsup_{\delta\to 0}\frac{\E[\tau_{\la}]}{\log (1/\delta)}
                &\le
                \sum_{i\in[\la]}
                \frac{1}{2\De_i^2}\com\label{cor11} 
                \\
                \limsup_{\delta\to 0}\frac{\E[\tau_{\mathrm{stop}}]}{\log (1/\delta)}
                &\le
                \sum_{i\in[K]}
                \frac{1}{2\De_i^2}\com\label{cor12}
            \end{align}
            \begin{align}
            &\E[\tau_{\la}]
            =
            \lo\left(
            \frac{\la \log \frac{1}{\delta} +(K-\lambda)\log\log \frac{1}{\delta}+K\log \frac{K}{\Delta}}{\De^2}
            \right)\com\label{cor21}\\ 
            &\E[\tau_{\mathrm{stop}}]
            =
            \lo\left(
            \frac{K \log (1/\delta)+K\log (K/\Delta) }{\De^2}
            \right)\per\label{cor22}
            \end{align}
        \end{corollary}

        \begin{proof}
        Since
        \small
        \begin{align}
        \limsup_{\de\to 0}\frac{n_i}{\log (1/\delta)}=\frac{1}{2(\De_i-\ep)^2}\com\n
        \end{align}
        \normalsize
        we obtain \eqref{cor11} and \eqref{cor12} by letting $\epsilon\downarrow 0$.
        We obtain \eqref{cor21} and \eqref{cor22} by letting
        $\ep=\Delta/2$ in Theorem \ref{thm_upper}.
        \end{proof}

        Note that
        $d(\mu_i,\xi)\ge 2(\mu_i-\xi)^2=2\Delta_i^2$ from Pinsker's inequality
        and its coefficient two cannot be improved.
        Thus we see that the upper bound in
        \eqref{cor11} in Corollary \ref{cor_upper}
        is almost 
        optimal in view of the lower bound in Theorem \ref{thm_lower}
        for sufficiently small $\delta$.
        The authors believe that the coefficient $2\De_i^2$ can be improved to $d(\mu_i,\xi)$
        by the techniques in the KL-UCB algorithm (Kullback-Leibler UCB, \citealp{KLUCB}) 
        and the Thompson sampling algorithm \citep{pmlr-v23-agrawal12},
        although we use the sampling strategy based on the UCB algorithm \citep{Auer2002} for simplicity of the analysis.
        Eq.~\eqref{cor21} means that 
        the sample complexity of $\E[\tau_{\la}]$ scales with
        $\lo(\la\log \frac{1}{\delta}+(K-\la) \log\log\frac{1}{\delta})$ for moderately small $\delta$,
        which is contrasted with the sample complexity $\lo(K \log \frac{1}{\delta})$
        for the best arm identification \citep{JMLR:v17:kaufman16a}.
        Furthermore, we see from \eqref{cor12} and \eqref{cor22}
        that the HDoC algorithm reproduces the optimal sample complexity for
        the thresholding bandits \citep{Locatelli:2016:OAT:3045390.3045569}. 

        {\bf Remark 1.}
        We can easily extend GAI in a Bernoulli setting to GAI in a Gaussian setting with known variance $\sigma^2$.
        In the proofs of Theorems \ref{thm_error} and \ref{thm_upper},
        we used the assumption of the Bernoulli reward only in Hoeffding's inequality
        expressed as
        \vspace{-3mm}
        \begin{align}
        \mathbb{P}[\hat{\mu}_{i,n} \le \mu_i - \ep] \le \e^{-2n\ep^2}\com\n
        \end{align}
        where $\muhat_{i,n}$ is the empirical mean of the rewards when arm $i$ has been pulled $n$ times.
        When each reward follows a Gaussian distribution with variance $\sigma^2$,
        the distribution of the empirical mean is evaluated as
        \begin{align}
        \mathbb{P}[\hat{\mu}_{i,n} \le \mu_i - \ep] \le \e^{-\frac{n\ep^2}{2\sigma^2}}\n
        \end{align}
        by Cram\'er's inequality.
        By this replacement
        the score of HDoC becomes   $\tilde{\mu}_i (t) = \hat{\mu}_i (t) + \sqrt{\frac{2\sigma^2 \log t}{N_i(t)}}$,
        the score of LUCB-G becomes $\overline{\mu}_i (t) = \hat{\mu}_i (t) + \sqrt{\frac{2\sigma^2 \log(4KN_i^2(t)/\delta)}{N_i(t)}}$ 
        and the score for identifying good arms becomes $\underline{\mu}_i (t) = \hat{\mu}_i (t) - \sqrt{\frac{2\sigma^2 \log(4KN_i^2(t)/\delta)}{N_i(t)}}$ 
        in a Gaussian setting given variance $\sigma^2$, 
        while the score of APT-G in a Gaussian setting is the same as the score of APT-G in a Bernoulli setting.
      
        {\bf Remark 2.} 
        Theorem 2 and the evalution of $\tau_\mathrm{stop}$ in Theorem 3 do not depend on
        the sampling strategy and only use the fact that the identification criterion is given by Lines 5--11 in Algorithm 1. 
        Thus, these results still hold even if we use the LUCB-G and APT-G algorithms.
 
        {\bf Remark 3.}
        The evaluation of the error probability is based
        on the union bound over all rounds $t\in \mathbb{N}$,
        and the identification criterion in Lines 5--11 in Algorithm \ref{Alg1}
        is designed for this evaluation.
        The use of the union bound does not worsen
        the asymptotic analysis for $\delta\to 0$
        and we use this identification criterion
        to obtain a simple sample complexity bound.
        On the other hand,
        it is known that
        the empirical performance can be considerably improved
        by, for example,
        the bound based on the law of iterated logarithm
        in \citet{pmlr-v35-jamieson14} that can avoid the union bound. 
        We can also use an identification criterion based on such a bound
        to improve empirical performance 
        but this does not affect the result of relative comparison
        since we use the same identification criterion between algorithms with different
        sampling strategies.

    \subsection{Gap between Lower and Upper Bounds}\label{gap}
        As we can see from Theorem \ref{thm_upper} and its proof,
        an arm $i>\lambda$ (that is, an arm other than top-$\lambda$ ones)
        is pulled roughly 
        $\lo( \frac{\log \log (1/\delta)}{\Delta_{\lambda,i}})$ times
        until HDoC outputs $\lambda$ good arms.
        On the other hand, the lower bound in Theorem \ref{thm_lower}
        only considers $\lo(\log \frac{1}{\delta})$ term and
        does not depend on arms $i>\lambda$.
        Therefore, in the case where
        $(K-\lambda)$ is very large compared to $\frac{1}{\delta}$
        (more specifically, in the case of
        $K-\lambda=\Omega\left(\frac{\log (1/\delta)}{\log \log (1/\delta)}\right)$),
        there still exists a gap between the lower bound in \eqref{lower}
        and the upper bound in \eqref{cor21}.  
        Furthermore, the bound in \eqref{cor21} becomes meaningless
        when $\Delta_{\la,\la+1}\approx 0$.
        In fact,
        the $\lo(\log \log \frac{1}{\delta})$ term for small $\Delta_{\la,\la+1}$
        is not negligible 
        in some cases as we will see experimentally in Section \ref{Exp}.
        
        To fill this gap, it is necessary to consider
        the following difference between the cumulative regret minimization
        and GAI.
        Let us consider the case of pulling two good arms with the same expected rewards.
        In the cumulative regret minimization,
        which of these two arms is pulled
        makes no difference in the reward
        and, for example, it suffices for
        pulling these two arms alternately.
        On the other hand in GAI, the agent should output
        one of these good arms as fast as possible; hence,
        it is desirable to pull one of these equivalent arms
        with a biased frequency.
        However,
        the bias in the numbers of samples between seemingly equivalent arms
        increases the risk to miss an actually better arm
        and this dilemma becomes a specific difficulty in GAI.
        The proposed algorithm, HDoC, is not designed to cope with
        this difficulty and,
        improving $\lo(\log \log \frac{1}{\delta})$
        term from this viewpoint is important future work.

\section{Numerical Experiments}\label{Exp}
    In this section we experimentally compare the performance of HDoC with that of LUCB-G and APT-G.
    In all experiments,
    each arm is pulled five times as burn-in
    and
    the results are the averages over
    1,000 independent runs.
    
    \subsection{Threshold Settings}

        We consider three settings named
        Threshold 1--3, which are based on
        Experiment 1-2 in \citet{Locatelli:2016:OAT:3045390.3045569}
        and Experiment 4 in \citet{DBLP:journals/corr/MukherjeeNSR17}.
               
        {\bf Threshold 1 (Three group setting):}
        Ten Bernoulli arms with mean $\mu_{1:3} = 0.1$, $\mu_{4:7} = 0.35 + 0.1 \cdot (0:3)$ and $\mu_{8:10} = 0.9$,
        and threshold $\xi = 0.5$, where $(i:j)$ denotes $\{i,i+1,i+2,\ldots,j-1,j\}$.
        
        {\bf Threshold 2 (Arithmetically progressive setting):}
        Six Bernoulli arms with mean $\mu_{1:6} = 0.1 \cdot (1:6)$ and
        threshold $\xi = 0.35$.
        
        {\bf Threshold 3 (Close-to-threshold setting):}
        Ten Bernoulli arms with mean $\mu_{1:3} = 0.55$ and $\mu_{4:10} = 0.45$ and threshold $\xi = 0.5$.

        \begin{table*}[h]
            \centering
            \caption{Averages and standard deviations of arm-pulls over 1000 independent runs in Threshold 1--3 and Medical 1--2 for $\delta=0.05$. \newline Symbol ``--'' denotes that the agent does not output arm $\lambda \in [K]$ or $\bot$ (NULL) within 100,000 arm-pulls in almost all the runs.}
            \scalebox{0.79}{
            \begin{tabular}{c|cccccccccc} \hline
                Thre. 1    & $\tau_1$           & $\tau_2$  & $\tau_3$   & $\tau_4$  & $\tau_5$   & $\tau_\mathrm{stop}$   &    \\ 
                HDoC            & \textbf{114.0 $\pm$ 21.8}   & \textbf{146.7 $\pm$ 22.6}     & \textbf{186.8 $\pm$ 34.4}   & \textbf{778.7 $\pm$ 741.1}    & \textbf{5629.2 $\pm$ 1759.6}  & 10264.1 $\pm$ 2121.1  \\ 
                LUCB-G          & 134.4 $\pm$  26.6   & 167.3  $\pm$ 27.5   & 197.0 $\pm$ 31.2       & \textbf{798.0 $\pm$ 246.5}   & \textbf{5702.9 $\pm$ 1589.9}   & 10258.2 $\pm$ 2054.9 \\ 
                APT-G            & 6067.5 $\pm$ 1789.3  & 6282.5 $\pm$ 1810.9  & 6473.4 $\pm$ 1820.9  & 8254.0  $\pm$ 1909.5  & 10161.3 $\pm$ 2062.0  & 10243.0 $\pm$ 2062.0 \\ \hline \hline
            
                Thre. 2 & $\tau_1$           & $\tau_2$           & $\tau_3$ &$\tau_\mathrm{stop}$ & \\
                HDoC         & \textbf{202.2 $\pm$ 106.7}   & 825.8 $\pm$ 1048.7     & \textbf{5237.2 $\pm$ 1614.8}   & 10001.2 $\pm$ 2051.5&  \\  
                LUCB-G         & 259.5 $\pm$ 120.4  & \textbf{763.7 $\pm$ 260.4}   & 5566.6 $\pm$ 1575.5   & 9961.8 $\pm$ 1957.7 &  \\
                APT-G          & 6891.3 $\pm$ 1776.4   & 7990.7 $\pm$ 1839.1 & 9971.4 $\pm$ 1976.5 & 10048.9 $\pm$ 1976.5 &   \\ \hline \hline

                Thre. 3 & $\tau_1$           & $\tau_2$           & $\tau_3$ &$\tau_\mathrm{stop}$ &   \\
                HDoC        & \textbf{7081.3 $\pm$ 2808.4}  & \textbf{10955.9 $\pm$ 2954.4}   & 18063.0 $\pm$ 9252.0       & 46136.6 $\pm$ 4699.4 \ &  \\
                LUCB-G         & 10333.0 $\pm$ 3330.0  & 14183.6  $\pm$ 3186.7    & \textbf{17162.8  $\pm$ 2997.7}  & 46059.1 $\pm$ 4740.4 &   \\
                APT-G &  44326.0 $\pm$ 4708.1 & 45212.9 $\pm$ 4727.2 & 45676.7 $\pm$ 4727.8 & 45852.7 $\pm$ 4727.8  &   \\ \hline  \hline 

                Med. 1 & $\tau_1$  & $\tau_\mathrm{stop}$                        & \multicolumn{1}{||c|}{Med. 2} & $\tau_1$   & $\tau_2$   & $\tau_3$ & $\tau_\mathrm{stop}$          \\ 
                HDoC  & \textbf{10170.1 $\pm$ 4276.6}    & 31897.1 $\pm$ 5791.0  &\multicolumn{1}{||c|}{HDoC}    &   \textbf{111.8 $\pm$ 67.9}  & \textbf{267.6 $\pm$ 125.1}   & --  & --     \\ 
                LUCB-G & 10524.9 $\pm$ 3189.7 & 31827.9 $\pm$ 5786.0             & \multicolumn{1}{||c|}{LUCB-G} &   \textbf{110.7 $\pm$ 66.2}  & \textbf{262.3 $\pm$ 117.9}   & --  & --     \\ 
                APT-G & 30847.1 $\pm$ 5724.9  &   31571.7 $\pm$ 5829.4           &\multicolumn{1}{||c|}{APT-G}  &  -- &  -- &  --  & --    \\ \hline 
            \end{tabular}\label{Result1}
            }

            \centering
            \caption{Averages and standard deviations
            of arm-pulls over 1000 independent runs in Threshold 1--3 and Medical 1--2 for $\delta=0.005$.}
            \scalebox{0.79}{
            \begin{tabular}{c|cccccccccc} \hline
                Thre. 1    & $\tau_1$           & $\tau_2$  & $\tau_3$   & $\tau_4$  & $\tau_5$   & $\tau_\mathrm{stop}$   &    \\ 
                HDoC            & \textbf{130.5 $\pm$ 28.6}   & \textbf{168.3 $\pm$ 25.9}     & \textbf{209.6 $\pm$ 37.1}   & \textbf{864.1 $\pm$ 518.1}    & \textbf{6183.0 $\pm$ 1956.0}  & 11464.8 $\pm$ 2238.4& \\
                LUCB-G            & 164.6  $\pm$ 31.4  & 200.4 $\pm$ 31.1   & 235.0 $\pm$ 34.0   & 926.6 $\pm$ 257.3   & 6357.5 $\pm$ 1696.9  & 11323.9 $\pm$  2202.6 & \\
                APT-G             & 7300.6 $\pm$ 1625.6 & 7471.0 $\pm$ 1642.3 & 7647.7 $\pm$ 1661.0 & 9433.1 $\pm$ 1902.8 & 11344.9 $\pm$ 2292.9 & 11417.3 $\pm$ 2294.5 & \\ \hline \hline
            
                Thre. 2 & $\tau_1$           & $\tau_2$           & $\tau_3$ &$\tau_\mathrm{stop}$ & \\
                HDoC         & \textbf{231.3 $\pm$ 122.3 }   & \textbf{897.4 $\pm$ 1012.0} & \textbf{5783.7 $\pm$ 1564.5}   & 11197.2 $\pm$ 2129.6 &  \\ 
                LUCB-G         & 301.1 $\pm$ 127.2 & \textbf{875.4 $\pm$ 265.9 }   & 6163.5 $\pm$ 1612.2  & 11099.1 $\pm$ 2054.0 & \\ 
                APT-G          & 8070.7 $\pm$ 1616.8 & 9159.3 $\pm$ 1726.7 & 11081.3 $\pm$  2067.4 & 11173.4 $\pm$ 2058.2 & \\ \hline \hline
             
                Thre. 3 & $\tau_1$ & $\tau_2$ & $\tau_3$ & $\tau_\mathrm{stop}$  &    \\ 
                HDoC       & \textbf{9721.1 $\pm$ 3396.3}  & \textbf{14019.9 $\pm$ 3154.4}   & \textbf{18711.2 $\pm$ 7128.7}  & 50937.3 $\pm$ 5008.2 & \\ 
                LUCB-G    & 11926.5 $\pm$ 3488.7 & 15943.1 $\pm$ 3145.2 & \textbf{18980.1 $\pm$ 3072.1} & 50700.7 $\pm$ 4803.3                                &   \\ 
                APT-G      & 49632.9 $\pm$ 5010.9 & 50458.2 $\pm$ 4925.2 & 50841.6 $\pm$ 4910.3 & 50989.3 $\pm$ 4905.0                               & \\ \hline \hline
 
                Med. 1   & $\tau_1$  & $\tau_\mathrm{stop}$                        &  \multicolumn{1}{||c|}{Med. 2}     & $\tau_1$ & $\tau_2$   & $\tau_3$ & $\tau_\mathrm{stop}$  &        \\  
                HDoC     & \textbf{11274.7 $\pm$ 4844.1}    & 34840.2 $\pm$ 5942.5 &  \multicolumn{1}{||c|}{HDoC}       &  \textbf{134.5 $\pm$ 71.1}  &  \textbf{310.1 $\pm$ 126.4} & -- & --  & \\ 
                LUCB-G   & 11739.9 $\pm$ 3339.7 & 34595.5 $\pm$ 5890.0             &  \multicolumn{1}{||c|}{LUCB-G}     &  \textbf{135.5 $\pm$ 72.5}  &  \textbf{315.0 $\pm$ 132.6} & -- & --  & \\  
                APT-G        & 34516.4 $\pm$ 5994.6 & 35189.0 $\pm$ 6055.4         &  \multicolumn{1}{||c|}{APT-G}      &  -- &  --  &  --  & -- & \\ \hline 
            \end{tabular}\label{Result2}
            }
            
        \end{table*}

    \subsection{Medical Settings}
      We also consider two medical settings of dose-finding in clinical trials as
      GAI. In general, the dose of a drug is quite important.
      Although high doses are usually more effective than low doses,
      low doses can be effective than high doses because high doses often cause
      bad side effects.
      Therefore, it is desirable to list various doses of a drug with satisfactory effect,
      which can be formulated as GAI.
      We considered two instances of the dose-finding problem based on \citet{seckinumab} and \citet{Liu2017} as Medical 1--2, respectively, specified as follows.
      In both settings, the threshold $\xi$ corresponds to the satisfactory effect.

      {\bf Medical 1 (Dose-finding of secukinumab for rheumatoid arthritis with satisfactory effect):}
      Five Bernoulli arms with mean $\mu_1=0.36$, $\mu_2=0.34$, $\mu_3=0.469$, $\mu_4=0.465$, $\mu_5=0.537$,
      and threshold $\xi=0.5$. 
      
      Here,
      $\mu_1, \mu_2,\ldots, \mu_5$ represent placebo, secukinumab 25mg, 75mg, 150mg and 300mg, respectively.
      The expected reward indicates American College of Rheumatology 20\% Response (ACR20) at week 16 given in \citet[Table 2]{seckinumab}.

      {\bf Medical 2 (Dose-finding of GSK654321 for rheumatoid arthritis with satisfactory effect):}
      Seven Gaussian arms with mean $\mu_1=0.5$, $\mu_2=0.7$, $\mu_3=1.6$, $\mu_4=1.8$, $\mu_5=1.2$, $\mu_6=1.0$ and $\mu_7=0.6$
      with variance $\sigma^2=1.44$ and threshold $\xi=1.2$.
    
      Here, $\mu_1, \mu_2,\ldots, \mu_7$ represent the positive effect\footnote{%
      The original values (smaller than zero)
      in \citet{Liu2017} represent the negative effect
      and we inverted the sign to denote the positive effect.} of      
      placebo, the dose of GSK654321 0.03, 0.3, 10, 20 and 30 mg/kg, respectively,
      where GSK654321 \citep{Liu2017} is a developing drug with nonlinear dose-response, which is based on the real drug GSK315234 \citep{choy}.
      The expected reward indicates change from the baseline in $\Delta$ Disease Activity Score 28 (DAS28) given in \citet[Profile 4]{Liu2017}. 
      The threshold $\xi = 1.2$ is based on \citet{article}.

%
    
    \subsection{Results}
        First we compare HDoC, LUCB-G and APT-G for
        acceptance error rates $\delta = 0.05,\,0.005$.
        Tables \ref{Result1} and \ref{Result2}
        show
        the averages and standard deviations of
        $\tau_1, \tau_2, \ldots, \tau_\lambda$ and $\tau_\mathrm{stop}$
        for these algorithms.
        In most settings, HDoC outperforms LUCB-G and APT-G.
        In particular, the number of samples required for APT-G is very large compared to those required for HDoC or LUCB-G,
        and the stopping times of all algorithms are close as discussed in Remark 2.
        The results verify that HDoC addresses GAI more efficiently than LUCB-G or APT-G.

        In Medical 2, 
        we can easily see $\tau_3, \tau_\mathrm{stop} =+\infty$ with high probability
        since the expected reward $\mu_5$ is equal to the threshold $\xi$.
        Moreover, APT-G fails to work completely,
        since it prefers to pull an arm whose expected reward is closest to the threshold $\xi$
        and selects the arm with mean $\mu_5$ almost all the times.
        In fact, Tables \ref{Result1}--\ref{Result2} show that APT-G cannot identify even one good arm within 100,000 arm-pulls
        whereas HDoC and LUCB-G can identify some good arms reasonably even in such a case.

        As shown in Tables \ref{Result1}--\ref{Result2},  
        the performance of HDoC is almost the same as that of LUCB-G in Medical 2, where 
        the expectations of the arms are very close to each other,
        taking the variance $\sigma^2$ into consideration.
        Figure \ref{plots1} shows the result of an experiment to
        investigate the behavior of HDoC and LUCB-G for Medical 2 in more detail,
        where $\tau_1,\tau_2$ are plotted for (possibly unrealistically) small $\delta$.
        Here ``Lower bound'' in the figure is the asymptotic lower bound $\sum_{i=1}^{\lambda}\frac{2\sigma^2 \log (1/\delta)}{\Delta_i^2}$
        of $\tau_{\lambda}$ for normal distributions (see Theorem \ref{thm_lower} and Remark 1).
        Since the result of HDoC asymptotically approaches to the lower bound,
        the $\lo(\log \frac{1}{\delta})$ term of the sample complexity of HDoC is almost optimal, 
        and the results show that the effect of $\lo(\log \log \frac{1}{\delta})$ term is not negligible for practical acceptance error rates such as $\delta=0.05$ and $ 0.005$.

        \begin{figure}[h] 
            \includegraphics[width=7.7cm]{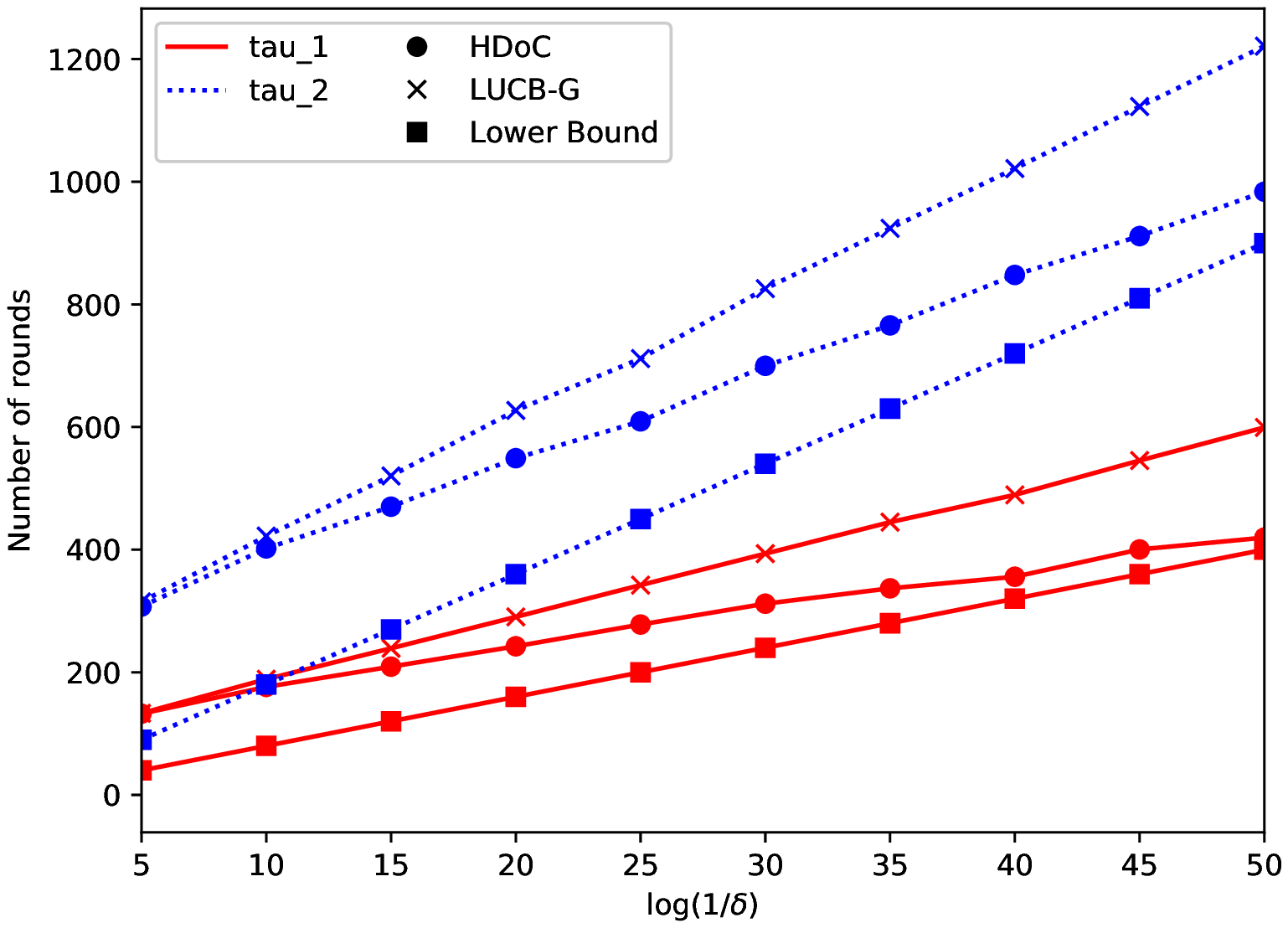}    
            \caption{Number-of-round plots of HDoC, LUCB-G and the lower bound for $\log \frac{1}{\delta}=5,10,\ldots,50$ in Medical 2.}
            \label{plots1}
          \end{figure}

       

\section{Proof of Theorem \ref{thm_lower}}\label{proof_lower_bound}
    In this section, we prove Theorem \ref{thm_lower} based on the following proposition on the expected number of samples
    to distinguish two sets of reward distributions.
    \begin{proposition}[Lemma 1 in \citealp{JMLR:v17:kaufman16a}]\label{prop_kaufmann}
        Let $\nu$ and $\nu'$ be two bandit models with $K$ arms such that for all $i$, 
        the distributions $\nu_i$ and $\nu_i'$ are mutually absolutely continuous. 
        For any almost-surely finite stopping time $\sigma$ and
        event $\calE$,
        \begin{align}
            \sum_{i=1}^K \E[N_i (\sigma)]\mathrm{KL}(\nu_i,\nu_i')\ge
            d(\mathbb{P}_\nu[\calE],\mathbb{P}_{\nu'}[\calE])\com\n
        \end{align}
        where 
        $\mathrm{KL}(\nu_i, \nu_j)$ is the  Kullback-Leibler divergence between distributions $\nu_i$ and $\nu_j$,
        and $d(x,y) = x\log(x/y)\allowbreak +(1-x)\log((1-x)/(1-y))$ is the binary relative entropy, with convention that $d(0,0)=d(1,1)=0$.
    \end{proposition}
    Standard proofs on the best arm identification problems
    set $\calE$ as an event such that $\mathbb{P}[\calE]\ge 1-\delta$ under
    any $\delta$-PAC algorithm.
    On the other hand,
    we leave $\mathbb{P}[\calE]$ to range from $0$ to $1$
    and establish a lower bound as a minimization problem over $\mathbb{P}[\calE]$.

    \begin{proof}[Proof of Theorem \ref{thm_lower}.]
        Fix $j\in [m]$ and consider a set of Bernoulli distributions $\{\nu_i'\}$ with expectations
    $\{\mu_i'\}$ given by
    \begin{align}
        \mu_i'=
        \begin{cases}
            \xi-\ep \com &\mbox{if $i=j$},\\
            \mu_i \com &\mbox{if $i\in [K]\setminus \{j\}$}.
        \end{cases}\n
    \end{align}
    Let $\calE_j=\{j\in   \{ \hat{a}_i\}_{i=1}^{\min \{\la,\hat{m}\}  }  \}$ and $p_j=\mathbb{P}\left[j\in \{\hat{a}_i\}_{i=1}^{\min \{\la,\hat{m}\}  } \right]$ under $\{\nu_i\}$.
    Since $j$ is not a good arm under $\{\nu_i'\}$,
    we obtain from Prop.~\ref{prop_kaufmann} that
    \small
    \begin{align}
        \E[N_j]d_j &\ge d(p_j,\min\{\delta,p_j\}) \nn
        &=\max \biggl\{ p_j \log\frac{1}{\min\{\delta,p_j\}}-h(p_j) \nn
        &\qquad\qquad +(1-p_j)\log\frac{1}{1-\min\{\delta,p_j\}},0 \biggr\} \nn
        &\ge \max\brx{p_j\log \frac{1}{\min\{\delta,p_j\}}-\log 2,\,0}\nn
        &\ge \max\brx{p_j\log \frac{1}{\delta}-\log 2,\,0}\com 
        \n 
    \end{align} 
    \normalsize
    where we set
    $d_i=d(\mu_i,\xi-\epsilon)$ and
    $h(p)=-p\log p-(1-p)\log(1-p)\le \log 2$ is the binary entropy function.

    Here note that
    \small
    \begin{align}
        \sum_{i=1}^m p_i
        &=
        \mathbb{E}_{\nu}[| [m]\cap \{\hat{a}_i\}_{i=1}^{\min \{\la,\hat{m}\}}|]\nn
        &\ge \lambda\mathbb{P}_{\nu}[
        \{\{\hat{a}_{i}\}_{i=1}^{\min\{\la,\hat{m}\}}\subset [m]\},\,\hat{m}\ge
        \lambda]\ge 
        \la(1-\delta)\n
    \end{align}
    \normalsize
    under any $(\lambda,\delta)$-PAC algorithm.
    Thus we have
    \small
    \begin{align}
        \sum_{i=1}^K\E[N_i] \ge \sum_{i=1}^m\E[N_i]\ge C^*\com\n
    \end{align}
    \normalsize
    where $C^*$ is the optimal value of the optimization problem
    \begin{align}
        (\mathrm{P_1}) \quad
        \minimize & \sum_{i=1}^m \frac{1}{d_i}\max\brx{p_i\log \frac{1}{\delta}-\log 2,\,0}, \nn
        \subjectto & \sum_{i=1}^mp_i\ge \lambda (1-\delta)\com\nn  
        & 0\le p_i\le 1\com\quad\forall i\in[m]\com\n 
    \end{align}
    which is equivalent to the linear programming problem
    \begin{align}
        (\mathrm{P_2})\quad
        \minimize & \sum_{i=1}^m \frac{x_i}{d_i}\com \qquad \qquad \qquad \qquad \qquad \quad \nn
        \subjectto & \sum_{i=1}^mp_i\ge \lambda (1-\delta)\com \nn
        & x_i\ge p_i\log \frac{1}{\delta}-\log 2\com\quad\forall i\in[m]\com\nn
        & 0\le p_i\le 1\com\quad x_i\ge 0\com\quad\forall i\in[m]\per\n
    \end{align}
    The dual problem of $(\mathrm{P}_2)$ is given by
    \begin{align}
        (\mathrm{P}_2')\quad
        \maximize & \lambda(1-\de)\alpha-(\log 2)\sum_{i=1}^m\beta_i-\sum_{i=1}^m \ga_i \nn
        \qquad\subjectto
        & \beta_i\le \frac{1}{d_i}\com\quad\forall i\in[m]\com\nn
        & \alpha- \be_i\log \frac{1}{\delta}-\ga_i\le 0\com\quad\forall i\in[m]\com\nn
        & \alpha,\beta_i,\ga_i\ge 0\com\quad\forall i\in[m]\per\n
    \end{align}
    Here consider the feasible solution
    of $(\mathrm{P}_2')$ given by
    \small
    \begin{align}
        \alpha&=\frac{1}{d_{\la}}\log\frac{1}{\delta}\com \qquad
        \beta_i=
        \begin{cases}
            \frac{1}{d_i},&i\le \la,\\
            \frac{1}{d_{\la}},&i> \la,
        \end{cases}\nn
        \ga_i&=
        \begin{cases}
            \pax{\frac{1}{d_{\la}}-\frac{1}{d_i}}\log \frac{1}{\delta},&i\le\la,\\
            0,&i> \la,
        \end{cases}\n
    \end{align}
    \normalsize
    which attains the objective function
    \vspace{-2mm}
    \small
    \begin{align} 
    &\frac{\la(1-\de)}{d_{\la}}\log \frac{1}{\delta}
    -(\log 2)\left(\sum_{i\le \la}\frac{1}{d_i}
    +\frac{m-\la}{d_{\la}}
    \right)
    \nn
    &\quad -\sum_{i\le \la} \pax{\frac{1}{\!d_{\la}}-\frac{1}{d_i}} \log \frac{1}{\delta}
    \nn
    &=
    \sum_{i\le \la}\left(\frac{1}{d_i}\log \frac{1}{\delta}\!-\!\frac{\log 2}{d_i}\right)\!
    -\frac{\la\de}{d_{\la}}\log \frac{1}{\de}
    -\frac{(m-\la)\log 2}{d_{\la}}\nn
    &\ge
    \sum_{i\le \la}\frac{1}{d_i}\log \frac{1}{2\delta}
    -\frac{\la}{d_{\la}}
    -\frac{(m-\la)\log 2}{d_{\la}}\nn
    &
    \qquad \qquad \qquad \quad \   \since{by $\sup_{0<\de\le 1} \delta\log (1/\delta) = 1/\e<1$}
    \nn
    &\ge
    \sum_{i\le \la}\frac{1}{d_i}\log \frac{1}{2\delta}
    -\frac{m}{d_{\la}}
    \per\n
    \end{align}
    \normalsize
    Since the objective function of a feasible solution
    for the dual problem $(\mathrm{P}_2')$ of
    $(\mathrm{P}_2)$ is always smaller than the optimal value $C^*$ of $(\mathrm{P}_2)$, we have
    \small
    \vspace{-1mm}
    \begin{align}
    \sum_{i=1}^m\E[N_i]
    &\ge C^*\nn
    &\ge \sum_{i\le \la}\frac{1}{d_i}\log \frac{1}{2\delta}
    -\frac{m}{d_{\la}}\nn
    &= \sum_{i\le \la}\frac{1}{d(\mu_i,\xi-\ep)}\log \frac{1}{2\delta}
    -\frac{m}{d(\mu_{\la},\xi-\ep)}\per\n
    \end{align}   
    \normalsize           
    We complete the proof by letting $\ep \downarrow 0$.
  \end{proof}

  \vspace{-2mm}
\section{Conclusion}
  In this paper, we considered and discussed a new multi-armed bandit problem called good arm identification (GAI).
  The objective of GAI is to minimize not only the total number of samples to identify all good arms
  but also the number of samples until
  identifying $\lambda$ good arms for each $\lambda=1,2,\dots$,
  where a good arm is an arm whose expected reward is greater than or equal to threshold $\xi$.
  Even though GAI, which is a pure-exploration problem, does not face the exploration-exploitation dilemma of reward,
  GAI encounters a new kind of dilemma: the exploration-exploitation dilemma of confidence.
  We derived a lower bound on the sample complexity of GAI, 
  developed an efficient algorithm, HDoC,
  and then we theoretically showed the sample complexity of HDoC almost matches the lower bound.
  We also experimentally demonstrated that HDoC outperforms algorithms based on other pure-exploration problems
  in the three settings based on the thresholding bandit
  and two settings based on the dose-finding problem in the clinical trials.
  
\newpage

\section*{Acknowledgements}
JH acknowledges support by KAKENHI 16H00881.
MS acknowledges support by KAKENHI 17H00757.
This work was partially supported by JST CREST Grant Number JPMJCR1662, Japan.

\bibliographystyle{icml2018}
\bibliography{icml2018}
 
\clearpage
\appendix

\section{Proof of Theorem \ref{thm_error}}\label{proof_error}
In this appendix we prove Theorem \ref{thm_error} based on the following lemma.
\begin{lemma}\label{lem_error}
  \begin{align}
    \mathbb{P}\left[\bigcup_{n \in \bbN}\{\omu_{i,n} < \xi\} \right] &\le \frac{\de}{K}\com \since{$\mathrm{for}$ $\mathrm{any}$ $i\in [m]$}\com \nn
    \mathbb{P}\left[\bigcup_{n \in \bbN}\{\umu_{i,n} \ge \xi\} \right] &\le \frac{\de}{K}\com\n \since{$\mathrm{for}$ $\mathrm{any}$ $i\in [K] \setminus [m]$} \per
  \end{align}
  \end{lemma}
  \begin{proof}
    For any $i\in [m]\com$
    \begin{align} 
      \lefteqn{\mathbb{P}\left[\bigcup_{n \in \bbN}\{\omu_{i,n} < \xi\} \right]}  \nn
      &\le \sum_{n \in \bbN} \mathbb{P}\left[\omu_{i,n}< \xi \right]  \phantom{www} \nn
      &\le \sum_{n \in \bbN} \mathbb{P}\left[\omu_{i,n}< \mu_i \right] \phantom{www} \since{by $\mu_i \geq \xi$ for $i \in [m]$} \nn
      &\le \sum_{n \in \bbN} \e^{-2n  \left( \sqrt {\frac{\log (4Kn^2 / \delta)}{2n}} \right)^2 } \since{by Hoeffding's inequality}  \nn
      &= \sum_{n \in \bbN} \frac{\delta}{4Kn^2}  \nn
      &= \frac{\pi^2 \delta}{24K} \phantom{www} \since{by $\displaystyle \sum_{n \in \mathbb{N}} \frac{1}{n^2} = \frac{\pi^2}{6}$ } \nn
      &\le \frac{\delta}{K} \per\n
    \end{align}
    For any $i\in [K] \setminus [m]\com$ the same argument holds.
  \end{proof}

  \begin{proof}[Proof of Theorem \ref{thm_error}.] 
    We show that HDoC is $(\lambda,\delta)$-PAC for arbitrary $\lambda\in[K]$.
    
    First we consider the case that there are more than or equal to
    $\lambda$ good arms and show
    \begin{align}
    \mathbb{P}\left[
    \{\hat{m}< \lambda\}\,\cup\,\bigcup_{i\in  \{\hat{a}_1, \hat{a}_2, \ldots, \hat{a}_\lambda \}  }\{\mu_i< \xi\}
    \right]\le \de
    \label{pac_toprove1}.
    \end{align}
    
    Since we are now considering the case $m\ge \lambda$,
    the event $\{\hat{m}< \lambda\}$ implies that
    at least one good arm $j\in[m]$ is regarded as a bad arm, that is,
    $\{\umu_{j,n} \le \xi\}$ occurs for some $j\in [m]$ and $n\in\mathbb{N}$.
    Thus we have
    \begin{align}
    \mathbb{P}[\hat{m}< \lambda]
    &\le
    \sum_{j\in [m]}
        \mathbb{P}\left[\bigcup_{n \in \bbN}\{\omu_{j,n} < \xi\} \right]
    \nn
    &\le
    \sum_{j\in [m]}
    \frac{\de}{K}\since{by Lemma \ref{lem_error}}
    \nn
    &\le
    \frac{m\de}{K}\per\label{eq_error_pac}
    \end{align}
    On the other hand, since the event $\bigcup_{i\in  \{\hat{a}_1, \hat{a}_2, \ldots, \hat{a}_\lambda \}  }\{\mu_i< \xi\}$
    implies that $j \in \{\hat{a}_i\}_{i=1}^{\lambda}$ for some
    bad arm $j\in [K]\setminus [m]$,
    we have
    \begin{align}
    \mathbb{P}\left[\bigcup_{i\in  \{\hat{a}_1, \hat{a}_2, \ldots, \hat{a}_\lambda \}  }\{\mu_i< \xi\}\right]
    &\le
    \sum_{j\in [K]\setminus [m]}\mathbb{P}[j \in \{\hat{a}_i\}_{i=1}^{\lambda}]\nn
    &\le
    \sum_{j\in [K]\setminus [m]}
        \mathbb{P}\left[\bigcup_{n \in \bbN}\{\umu_{j,n} \ge \xi\} \right]
    \nn
    &\le
    \frac{(K-m)\de}{K}\label{eq_error_pac2}
    \end{align}
    in the same way as \eqref{eq_error_pac}.
    We obtain \eqref{pac_toprove1} by putting \eqref{eq_error_pac} and \eqref{eq_error_pac2} together.

    Next we consider the case that
    the number of good arms $m$ is less than $\lambda$
    and show
    \begin{align}
    \mathbb{P}[\hat{m}\ge \lambda]\le \de\per\label{pac_toprove2}
    \end{align}
    Since there are at most $m<\lambda$ good arms,
    the event $\{\hat{m}\ge \lambda\}$ implies that
    $j \in \{\hat{a}_i\}_{i=1}^{\lambda}$ for some $j\in [K]\setminus[\lambda]$.
    Thus, in the same way as \eqref{eq_error_pac2} we have
    \begin{align}
    \mathbb{P}\left[\hat{m}\ge \lambda\right]
    &\le
    \sum_{j\in [K]\setminus [m]}\mathbb{P}[j \in \{\hat{a}_i\}_{i=1}^{\lambda}]\nn
    &\le
    \frac{(K-m)\delta}{K}\nn
    &\le
    \de\com\n
    \end{align}
    which proves \eqref{pac_toprove2}.
    \end{proof}
    
\section{Proof of Theorem \ref{thm_upper}}\label{appndx_upper}
In this appendix, we prove Theorem \ref{thm_upper} based on the following lemmas,
and we define $T=K\max_{i\in [K]}\fl{n_i}\label{def_T}$.

\begin{lemma}\label{lem_ni}
If $n\ge n_i$ then
\begin{align}
\mathbb{P}[
\umu_{i,n}\le \xi]
&\le \e^{-2n\ep^2},
\qquad  \forall i\in[m] \com\label{lem_notstop1}\\
\mathbb{P}[
\omu_{i,n}\ge \xi]
&\le \e^{-2n\ep^2},
\qquad \forall i\in[K]\setminus [m]\per\label{lem_notstop2}
\end{align}
\end{lemma}
\begin{proof}
We only show \eqref{lem_notstop1} for
$i\in [m]$.
Eq.~\eqref{lem_notstop2} for $i\in [K]\setminus[m]$ is exactly the same.
From Hoeffding's inequality it suffices to show that
for $n\ge n_i$
\begin{align}
\sqrt{\frac{\log (4Kn^2/\de)}{2n}}\le
\mu_i-\xi-\ep=\Delta_i-\ep\per\n
\end{align}
We write $c=(\De_i-\ep)^2\le 1$
in the following for notational simplicity.
Then we can express $n\ge n_i$ as
\begin{align}
n=\frac{1}{c}
\log\frac{4t\sqrt{K/\de}}{c}\n
\end{align}
for some $t>\log\frac{5\sqrt{K/\de}}{c}>\log (5\sqrt{2})> 1$.
Then
\begin{align}
\lefteqn{
\sqrt{\frac{\log (4Kn^2/\de)}{2n}}\le
\De_i-\ep
}\nn
&\liff
\log (4Kn^2/\de)\le 2cn \nn
&\liff
\log
\frac{4K\pax{\log\frac{4t\sqrt{K/\de}}{c}}^2}{c^2\de}
\le \log \frac{16t^2 K}{c^2\de} \nn
&\liff
\log\frac{4t\sqrt{K/\de}}{c}
\le 2t
\nn
&\lif
\log\frac{4\sqrt{K/\de}}{c}+t-1
\le 2t
\since{by $\log x\le x-1$}
\nn
&\liff
\log\frac{4\sqrt{K/\de}}{\e c}
\le t\per
\label{tobesatisfied}
\end{align}
We obtain the lemma since $t>\log\frac{5\sqrt{K/\de}}{c}$
satisfies \eqref{tobesatisfied}.
\end{proof}

\begin{lemma}\label{lem_samples}
\begin{align}
\Ex{\sum_{n=1}^{\infty}
\id[
\umu_{i,n}\le \xi]}
&\le
n_i+\frac{1}{2\ep^2}\com
\qquad \forall i\in[m]\com
\nn
\Ex{\sum_{n=1}^{\infty}
\id[
\omu_{i,n}\ge \xi]}
&\le
n_i+\frac{1}{2\ep^2}\com
\qquad \forall i\in[K]\setminus [m]\per
\n
\end{align}
\end{lemma}
\begin{proof}

  If arm $i \in[m]$ then
  \begin{align}
    \Ex{\sum_{n=1}^{\infty}\id[\umu_{i,n}\le \xi]}
      &\leq \Ex{ \sum_{n=1}^{ n_i } 1  +  \sum_{n= n_i +1  }^\infty \id[\umu_{i,n} \le \xi] }\nn
      &\leq n_i  +  \sum_{n=1  }^\infty \mathbb{P} [\umu_{i,n} < \xi] \nn
      &\leq n_i  +  \sum_{n=1  }^\infty  \e^{-2n\ep^2} \nn
      &\leq n_i + \frac{1}{\e^{2\ep^2} -1} \nn
      &\leq n_i + \frac{1}{2\ep^2} \per \n
  \end{align}

  If arm $i \in[K] \setminus [m]$ then
  \begin{align}
    \Ex{\sum_{n=1}^{\infty}\id[\omu_{i,n}\ge \xi]}
      &\leq \Ex{ \sum_{n=1}^{ n_i } 1  +  \sum_{n= n_i +1  }^\infty \id[\omu_{i,n} \geq \xi] }\nn
      &\leq n_i  +  \sum_{n=1  }^\infty \mathbb{P} [\omu_{i,n} \geq \xi] \nn
      &\leq n_i  +  \sum_{n=1  }^\infty  \e^{-2n\ep^2} \nn
      &\leq n_i + \frac{1}{\e^{2\ep^2} -1} \nn
      &\leq n_i + \frac{1}{2\ep^2} \per \n
  \end{align}      
\end{proof}

\begin{lemma}\label{lem_dec1}
  \begin{align}
  &\Ex{
  \sum_{t=1}^{\infty}
  \idx{a(t)\in [\la]}
  +
  \sum_{t=1}^{T}
  \idx{a(t)\notin [\la],\,\tilde{\mu}^*(t) \geq \mu_{\la} - \epsilon}
  } \nn
  &\le
  \sum_{i\in [\la]}n_i
  +
  \sum_{i\in [K]\setminus[\la]}
  \frac{\log T}{2(\De_{\la,i} - 2\epsilon)^2} +  \frac{K}{2\epsilon^2}\per\label{dec1_first}
  \end{align}
  \end{lemma}
  \begin{proof}
  For the first term of \eqref{dec1_first} we have
  \begin{align}
  \sum_{t=1}^{\infty}
  \idx{a(t)\in [\la]}
  &\le
  \sum_{i\in[\la]}
  \sum_{t=1}^{\infty}
  \sum_{t=1}^{\infty}
  \idx{a(t)=i}\nn
  &=
  \sum_{i\in[\la]}
  \sum_{t=1}^{\infty}
  \sum_{n=1}^{\infty}
  \idx{a(t)=i,\,N_i(t)=n}\per\label{ttoi}
  \end{align}
  Since the event $\{a(t)=i,\,N_i(t)=n\}$ occurs for at most one $t\in \bbN$
  we have
  \begin{align}
  &\sum_{t=1}^{\infty}
  \sum_{n=1}^{\infty}
  \idx{a(t)=i,\,N_i(t)=n} \nn
  &\le
  \sum_{n=1}^{\infty}
  \idx{\bigcup_{t=1}^{\infty}\{a(t)=i,\,N_i(t)=n\}}\nn
  &\le
  \sum_{n=1}^{\infty}
  \idx{\umu_{i,n}\le \xi}\per\label{use3}
  \end{align}
  By combining \eqref{ttoi} and \eqref{use3} with Lemma \ref{lem_samples}
  we obtain
  \begin{align}
  \Ex{
  \sum_{t=1}^{\infty}
  \idx{a(t)\in [\la]}
  }
  \le
  \sum_{i\in[\la]}\left(n_i+\frac{1}{2\ep^2}\right)\per\label{dec1_res1}
  \end{align}
    
  Next we consider the second term of \eqref{dec1_first}.
  By using the same argument as \eqref{ttoi} we obtain for $i\notin [\la]$ that
  \begin{align}
  \lefteqn{
  \sum_{t=1}^{T} \id [a(t) = i,\,\tilde{\mu}^*(t) \geq \mu_{\la} - \epsilon]
  }\nn
              &\leq \sum_{n=1}^{T} \idx{ \bigcup_{t=1}^{T} \{ a(t) = i,\,\tilde{\mu}^*(t) \geq \mu_{\la} - \epsilon,\,N_i(t)=n \} } \nn
              &\leq \sum_{n=1}^{T} \idx{ \bigcup_{t=1}^{T} \{ \tilde{\mu}^*(t) = \tilde{\mu}_i(t) \geq \mu_{\la} - \epsilon,\,N_i(t)=n \} } \nn
              &\leq \sum_{n=1}^{T} \idx{\hat{\mu}_{i,n} +\sqrt{\frac{\log T}{2n}} \geq \mu_{\la} - \epsilon}\nn
              &\leq \sum_{n=1}^{T} \idx{\hat{\mu}_{i,n} +\sqrt{\frac{\log T}{2n}} \geq \mu_i + \De_{\la,i} - \epsilon}    \nn  
              &\leq  \sum_{n=1}^{\frac{\log T}{2(\De_{\la,i} - 2\epsilon)^2}} 1  \nn
              &\quad+ \sum_{n=\frac{\log T}{2(\De_{\la,i} - 2\epsilon)^2}+1}^{T} \id \Bigg[ \hat{\mu}_{i,n} + \sqrt{\frac{\log T}{2\cdot \frac{\log T}{2(\De_{\la,i} - 2\epsilon)^2}}} \nn
              & \phantom{wwwwwwwwwwwwww} \geq \mu_i + \De_{\la,i} - \epsilon \Bigg] \nn
            &=  \frac{\log T}{2(\De_{\la,i} - 2\epsilon)^2} + \sum_{n=\frac{\log T}{2(\De_{\la,i} - 2\epsilon)^2}+1}^{T} \idx{\hat{\mu}_{i,n} \geq \mu_i + \epsilon}\per\n
  \end{align}
  By taking the expectation
  we have
  \begin{align}
  \lefteqn {\Ex{
  \sum_{t=1}^{T} \id [a(t) = i, \tilde{\mu}^*(t) \geq \mu_{\la} - \epsilon]
  } }\nn
              &\leq \frac{\log T}{2(\De_{\la,i} - 2\epsilon)^2} + \sum_{n=1}^{\infty}\mathbb{P}[\hat{\mu}_{i,n} \geq \mu_i + \epsilon]\nn
              &\leq \frac{\log T}{2(\De_{\la,i} - 2\epsilon)^2} + \sum_{n=1}^{\infty} \mathrm{e}^{-2n\epsilon^2}\label{hoeff1}\\
              &=  \frac{\log T}{2(\De_{\la,i} - 2\epsilon)^2} +  \frac{1}{\mathrm{e}^{2\epsilon^2} -1}  \nn
              &\leq \frac{\log T}{2(\De_{\la,i} - 2\epsilon)^2} +  \frac{1}{2\epsilon^2}\com
  \label{dec1_res2}
  \end{align}
  where \eqref{hoeff1} follows from Hoeffding's inequality.
  We complete the proof by combining \eqref{dec1_res1} with \eqref{dec1_res2}.
  \end{proof}
  
  \begin{lemma}\label{lem_ift}
  \begin{align}
  \Ex{\sum_{t=T+1}^{\infty}
  \idx{t\le \tau_{\la}}
  }
  \le
  \frac{K^{2-\frac{\ep^2}{(\min_{i \in [K]} \De_i-\ep)^2}}}{2\ep^2}
  \per\n
  \end{align}
  \end{lemma}
  \begin{proof}
  Note that at the $t$-th round
  some arm is pulled at least $\cl{(t-1)/K}$ times.
  Furthermore,
  $N_i(t)\ge \cl{(t-1)/K}$ implies that
  the arm $i$ is still in $\calA(t)$ when
  the arm $i$ is pulled $\cl{(t-1)/K}-1$ times.
  Thus we have  
  \begin{align}
  \lefteqn{
  \sum_{t=T+1}^{\infty}\idx{t\le\tau_{\la}}
  }\nn
        &\leq \sum_{i=1}^m \sum_{t=T+1}^\infty \idx{N_i(t)\ge \cl{(t-1)/K},\,
  t\le\tau_{\la}}\nn
        &\leq
  \sum_{i=1}^m \sum_{t=T+1}^\infty \idx{
  \umu_{i,\cl{(t-1)/K}-1}\le \xi
  } \nn
  &\quad+
  \sum_{i=m+1}^K \sum_{t=T+1}^\infty \idx{
  \omu_{i,\cl{(t-1)/K}-1}\ge \xi 
  }\per\label{notstop}
  \end{align}
  From the definition of $T=K\max_{i\in [K]}\fl{n_i}$, 
  we have $\cl{(t-1)/K}-1\ge n_i$ for all $i\in [K]$.
  Thus, the expectation of \eqref{notstop} is bounded by Lemma \ref{lem_ni}
  as 
  \begin{align}
  \lefteqn{ 
  \Ex{\sum_{t=T+1}^{\infty}\idx{t\le\tau_{\la}}}
  }\nn
        &\leq
  K
  \sum_{t=T+1}^{\infty}
  \e^{-2\ep^2(\cl{(t-1)/K}-1)}\nn
  &\le
  K
  \sum_{t=T+1}^{\infty}
  \e^{-2\ep^2((t-1)/K-1)}\nn
  &=
  \frac{K\e^{-2\ep^2((T-1)/K-1)}}
  {\e^{2\ep^2/K}-1} \nn
  &\phantom{wwwwwwwww}\since{by $\sum_{i=1}^{\infty}\e^{-ai}=\frac{1}{\e^{a}-1}$ for $a>0$}
  \nn
  &\le
  \frac{K\e^{-2\ep^2(\max_i n_i-2)}}
  {\e^{2\ep^2/K}-1}\nn
  &\phantom{wwwwwww}\since{by $T=K\max_{i\in [K]}\fl{n_i}$} 
  \nn
  &\le
  \frac{K\e^{-\frac{2\ep^2}{(\min_{i \in [K]} \De_i-\ep)^2}(\log\sqrt{K})}}   
  {\e^{2\ep^2/K}-1}
  \label{log2}\\
  &=
  \frac{K^{1-\frac{\ep^2 (\log \sqrt{K})}{(\min_{i \in [K]} \De_i-\ep)^2}}}
  {\e^{2\ep^2/K}-1}
  \nn
  &\le
  \frac{K^{2-\frac{\ep^2}{(\min_{i \in [K]} \De_i-\ep)^2}}}{2\ep^2}
  \com\n
  \end{align}
  where \eqref{log2} follows from
  \begin{align}
  n_i
  &\ge \frac{1}{( \De_i-\ep)^2}\log \left(4\sqrt{K}\log 5\sqrt{K}\right)\nn
  &\ge \frac{\log \sqrt{K}}{( \De_i-\ep)^2}+\log\left(4\log 5\sqrt{2}\right)\nn
  &\ge \frac{\log \sqrt{K}}{( \De_i-\ep)^2}+2.05\per\n
  \end{align}
  \end{proof}

  \begin{lemma}\label{lem_bef}
  \begin{align}
  \lefteqn{\Ex{
  \sum_{t=1}^{\infty}
  \idx{a(t)\notin [\la],\,t\le \tau_{\la},\,\tilde{\mu}^*(t) < \mu_{\la} - \epsilon}
  } } \nn
  &\le
  \frac{3K}{4\epsilon^2} + \frac{K\log\frac{1}{2\epsilon^2}}{4\epsilon^2}
  +\de\sum_{i\in [K]\setminus[\la]}
  n_i\n
  \end{align}
  \end{lemma}
  \begin{proof}
  The summation is decomposed into
  \begin{align}
  \lefteqn{
  \sum_{t=1}^{\infty}
  \idx{a(t)\notin[\la],\,t\le \tau_{\la},\,\tilde{\mu}^*(t) < \mu_{\la} - \epsilon}
  }\nn
              &\le \sum_{t=1}^{T} \id [
  \tilde{\mu}^*(t) < \mu^* - \epsilon,\,[\la]\cap \calA(t)\neq \emptyset]\nn
  &\quad+ \sum_{t=1}^{T} \id [a(t)\notin [\la],\,t\le\tau_{\la},\,[\la]\cap \calA(t)= \emptyset]
  \com
  \label{dec1}
  \end{align}
  where $\calA(t)=\{i\in[K]: \umu_i(t) < \xi\le\omu_i(t)\}$.
  From definition $\mut^*(t)=\max_{i\in\calA(t)}\mut_i(t)$
  the first term of \eqref{dec1} is evaluated as
  \begin{align}
  \lefteqn{
  \sum_{t=1}^{T}
  \idx{\tilde{\mu}^*(t) < \mu_{\la} - \epsilon,\,
  [\la]\cap \calA(t)\neq \emptyset
  }
  }\nn
  &\le
  \sum_{i\in[\la]}
  \sum_{t=1}^{T}
  \idx{\tilde{\mu}^*(t) < \mu_{\la} - \epsilon,\,
  i\in \calA(t)
  }\nn
              &\leq \sum_{i=1}^{\la}\sum_{t=1}^T
  \idx{\tilde{\mu}_{i}(t) < \mu_{\la} - \epsilon}\nn
              &\leq \sum_{i=1}^{\la}\sum_{n=1}^T \sum_{t=1}^T \id [\tilde{\mu}_{i}(t) < \mu_{\la} - \epsilon, N_{i}(t)=n]\nn
              &\leq \sum_{i=1}^{\la}\sum_{n=1}^T \sum_{t=1}^T
  \idx{\muhat_{i}(t) +\sqrt{\frac{\log t}{2n}}< \mu_{\la} - \epsilon, N_{i}(t)=n}\nn
              &\leq \sum_{i=1}^{\la}\sum_{n=1}^T \sum_{t=1}^T
  \id [t < \mathrm{e}^{2n(\hat{\mu}_{i,n} - \mu_{\la} +\epsilon)^2} ,\hat{\mu}_{i,n} < \mu_{\la} - \epsilon]\nn
              &\leq \sum_{i=1}^{\la}\sum_{n=1}^T \mathrm{e}^{2n(\hat{\mu}_{i,n} - \mu_{\la} +\epsilon)^2} \id [\hat{\mu}_{i,n} < \mu_{\la} - \epsilon]\per\label{termb}
  \end{align}
  Let $P_{i,n}(x)=\mathbb{P}[\hat{\mu}_{i,n}<x]$.
  Then the expectation of the inner summation of \eqref{termb} is bounded by
  \begin{align}
  \lefteqn{
  \sum_{n=1}^T
  \mathbb{E}[\mathrm{e}^{2n(\hat{\mu}_{i,n} - \mu_{\la} +\epsilon)^2} \id [\hat{\mu}_{i,n} < \mu_{\la} - \epsilon]]
  }\nn
              &\le
  \sum_{n=1}^T \int_{0}^{\mu_{\la}-\epsilon}
  \mathrm{e}^{2n(x - \mu_{\la} +\epsilon)^2} \rd P_{i,n}(x)\nn
              &=
  \sum_{n=1}^T \bigg( [\mathrm{e}^{2n(x - \mu_{\la} +\epsilon)^2}P_{i,n}(x)]_{0}^{\mu_{\la}-\epsilon} \nn
  &\quad- 4n\int_{0}^{\mu_{\la}-\epsilon}
  (x-\mu_{\la}+\ep)\mathrm{e}^{2n(x - \mu_{\la} +\epsilon)^2}
  P_{i,n}(x) \rd x \bigg)\nn
  &\phantom{wwwwwwwwwwwwwwwww}
  \since{by integration by parts}
  \nn
              &\leq
  \sum_{n=1}^T \bigg( \mathrm{e}^{-2n\epsilon^2} \nn
  &\, -4n\int_{0}^{\mu_{\la}-\epsilon} (x-\mu_{\la}+\epsilon) \mathrm{e}^{2n(x - \mu_{\la} +\epsilon)^2}\mathrm{e}^{-2n(x-\mu_{\la})^2} \rd x \bigg) \nn
  &\phantom{wwwwwwwwwwwwww}  \since{by Hoeffding's inequality} \nn
  \nn
              &=
  \sum_{n=1}^T \bigg( \mathrm{e}^{-2n\epsilon^2}  \nn
  &\quad+ 4n\mathrm{e}^{-2n\epsilon^2} \int_{0}^{\mu_{\la}-\epsilon} (\mu_{\la} -\epsilon -x)\mathrm{e}^{4n\epsilon(x - \mu_{\la} +\epsilon)} \rd x \bigg)\nn
              &\le
  \sum_{n=1}^{\infty}
  \mathrm{e}^{-2n\epsilon^2}\left(1 +\frac{1}{4n\epsilon^2} \right)
  \nn
              &= \frac{1}{\mathrm{e}^{2\epsilon^2}-1} + \frac{-\log(1-\mathrm{e}^{-2\epsilon^2})}{4\epsilon^2}\nn
              & \phantom{wwwwwwwwwwwww} \since{by $-\log(1-x)=\sum_{n=1}^{\infty}\frac{x^n}{n}$}\nn
              &= \frac{1}{\mathrm{e}^{2\epsilon^2}-1} + \frac{2\epsilon^2+\log(\frac{1}{\mathrm{e}^{2\epsilon^2}-1})}{4\epsilon^2}\nn
              &\leq \frac{1}{2\epsilon^2} + \frac{1}{2} + \frac{\log\frac{1}{2\epsilon^2}}{4\epsilon^2}\nn
              &\leq \frac{5}{8\epsilon^2}+ \frac{\log\frac{1}{2\epsilon^2}}{4\epsilon^2}\since{by $\ep<\frac12$}.\nn
  \label{termb2}
  \end{align}
  Combining \eqref{termb} with \eqref{termb2} we obtain
  \begin{align}
  &\Ex{
  \sum_{t=1}^{T} \id [
  \tilde{\mu}^*(t) < \mu^* - \epsilon,\,[\la]\cap \calA(t)\neq \emptyset]} \nn
  &\quad\le
  \frac{5K}{8\epsilon^2} + \frac{K\log\frac{1}{2\epsilon^2}}{4\epsilon^2}\per\label{dec11}
  \end{align}

  Next we bound the second term of \eqref{dec1}.
  Note that $\{t\le \tau_{\la},\,[\la]\cap \calA(t)= \emptyset\}$
  implies that
  $\omu_j(t')\le \xi$ occured for
  some $j\in [\la]$ and $t'<t$.
  Thus we have
  \begin{align}
  \lefteqn{
  \sum_{t=1}^{T} \id [a(t)\notin[\la],\,t\le\tau_{\la},\,[\la]\cap \calA(t)= \emptyset]
  }\nn
  &\le
  \sum_{i\in [K]\setminus[\la]}
  \sum_{t=1}^{T}
  \idx{a(t) = i,\,
  \bigcup_{j\in[\la]}\bigcup_{t'<t}\{\omu_j(t')\le \xi\}
  }\nn
  &\le
  \sum_{i\in [K]\setminus[\la]}
  \sum_{t=1}^{T}
  \idx{a(t) = i}
  \pax{
  \sum_{j\in[\la]}
  \idx{\bigcup_{t}\{\omu_j(t)\le \xi\}}
  }
  \nn
  &\le
  \sum_{i\in [K]\setminus[\la]}
  \sum_{n=1}^{T}
  \idx{
  \omu_{i,n}\ge \xi
  }
  \pax{
  \sum_{j\in[\la]}
  \idx{\bigcup_{t}\{\omu_j(t)\le \xi\}}
  }\com
  \label{ttoi2}   
  \end{align}
  where we used the same argument as \eqref{use3}
  in \eqref{ttoi2}.
  We can bound the expectation of \eqref{ttoi2} by Lemma \ref{lem_error} and \ref{lem_samples}
  as
  \begin{align}
  \lefteqn{
  \sum_{t=1}^{T} \id [a(t)\notin[\la],\,t\le\tau_{\la},\,[\la]\cap {\calA}(t)= \emptyset]
  }\nn
  &\le
  \frac{\la\de}{K}
  \left(
  \sum_{i\in [K]\setminus[\la]}
  n_i+\frac{K-\la}{2\ep^2}\right)\nn
  &\le
  \de\sum_{i\in [K]\setminus[\la]}
  n_i+\frac{K}{8\ep^2}
  \since{by $\la(K-\la)\le K^2/4$}\per
  \label{dec12}
  \end{align}
  We obtain the lemma by putting \eqref{dec1}, \eqref{dec11} and \eqref{dec12} together.
  \end{proof}

    \begin{proof}[Proof of Theorem \ref{thm_upper}.]\label{proof_upper_bound}
    The stopping time is decomposed into
    \begin{align}
    \tau_{\la}
    &=
    \sum_{t=1}^{\infty}
    \idx{a(t)\in [\la],\,t\le \tau_{\la}}
    +
    \sum_{t=1}^{\infty}
    \idx{a(t)\notin [\la],\,t\le \tau_{\la}}
    \nn
    &\le
    \sum_{t=1}^{\infty}
    \idx{a(t)\in [\la]}\nn
    &\quad+
    \sum_{t=1}^{\infty}
    \idx{a(t)\notin [\la],\,t\le \tau_{\la},\,\tilde{\mu}^*(t) \geq \mu_{\la} - \epsilon}\nn
    &\quad+
    \sum_{t=1}^{\infty}
    \idx{a(t)\notin [\la],\,t\le \tau_{\la},\,\tilde{\mu}^*(t) < \mu_{\la} - \epsilon}
    \nn
    &\le
    \sum_{t=1}^{\infty}
    \idx{a(t)\in [\la]}\nn
    &\quad+
    \sum_{t=1}^{T}\!
    \idx{a(t)\notin [\la],\,\tilde{\mu}^*(t)\! \geq \mu_{\la} - \epsilon}
    +\!\!\sum_{t=T+1}^{\infty}\!\idx{t\le \tau_{\la}}\nn
    &\quad+
    \sum_{t=1}^{\infty}
    \idx{a(t)\notin [\la],\,t\le \tau_{\la},\,\tilde{\mu}^*(t) < \mu_{\la} - \epsilon}
    \per\n
    \end{align}
    Lemmas \ref{lem_dec1}--\ref{lem_bef} bound the expectation of these terms,
    which complete the proof.
    \end{proof}

\end{document}